\documentclass[preprint]{elsarticle}

\usepackage{lineno,hyperref}
\modulolinenumbers[5]

\journal{Artificial Intelligence}

\usepackage{numcompress}
\usepackage{soul}
\usepackage{url}

\usepackage[utf8]{inputenc}
\usepackage{graphicx}
\usepackage{amsmath}
\usepackage{booktabs}
\usepackage{amsfonts}
\usepackage{amsthm}
\usepackage{amssymb}
\usepackage{bbm}
\usepackage{bm}
\usepackage[ruled,linesnumbered]{algorithm2e}
\usepackage{algorithmic}
\usepackage{color}
\usepackage{multirow}

\newtheorem{theorem}{Theorem}

\newtheorem{lemma}[theorem]{Lemma}

\newtheorem{definition}{Definition}

\usepackage{booktabs}
\usepackage{amsmath}
\usepackage{graphicx}
\usepackage{subfigure}
\usepackage{bm}
\usepackage{bbm}
\usepackage{array}
\usepackage{tabularx}
\usepackage{amssymb}
\usepackage{mathtools}
\usepackage{makecell, multirow}
\usepackage[ruled,linesnumbered]{algorithm2e}
\usepackage{algorithmic}
\usepackage{paralist}

\newcommand{\calA}{\mathcal{A}}

\newcommand{\calS}{\mathcal{S}}
\newcommand{\calG}{\mathcal{G}}
\newcommand{\calC}{\mathcal{C}}

\newcommand{\bbR}{\mathbb{R}}
\newcommand{\bbE}{\mathbb{E}}
\newcommand{\Renyi}{R\'enyi }

\newcommand{\dx}{{\text{d}x}}
\newcommand{\dy}{{\text{d}y}}
\newcommand{\bTk}{{$\beta$-Top-$k$ }}
\newcommand{\btk}{{$\beta$-top-$k$ }}

\newcommand{\bg}{\boldsymbol g}
\newcommand{\bh}{\boldsymbol h}
\newcommand{\bv}{\boldsymbol v}
\newcommand{\bz}{\boldsymbol z}
\newcommand{\bx}{\boldsymbol x}
\newcommand{\ourbm}{\boldsymbol m}

\newcommand{\bzero}{\boldsymbol 0}
\newcommand{\btx}{\boldsymbol{\tilde x}}
\newcommand{\btg}{\boldsymbol{\grave g}}
\newcommand{\tp}{{\tilde p}}
\newcommand{\vD}{{R}}
\newcommand{\tm}{{\tilde m}}
\newcommand{\btm}{\boldsymbol{\tilde m}}
\newcommand{\bdelta}{\boldsymbol{\delta}}
\newcommand{\ourmethod}{R\'enyi-Robust-Smooth}

\newcommand{\argmax}{\operatorname*{argmax}}

\usepackage{xcolor}

\newcommand\ao[1]{{\color{blue} \footnote
{\color{blue}Ao: #1}} }

\begin{document}

\begin{frontmatter}
\title{Certifiably Robust Interpretation via \Renyi Differential Privacy}
\author[mymainaddress]{Ao Liu}\corref{mycorrespondingauthor}\cortext[mycorrespondingauthor]{Corresponding author}\ead{liua6@rpi.edu}
\author[mysecondaryaddress]{Xiaoyu Chen}
\author[mythirdaddress]{Sijia Liu}
\author[mymainaddress]{Lirong Xia}
\author[myforthaddress]{Chuang Gan}
\address[mymainaddress]{Department of Computer Science, Rensselaer Polytechnic Institute, Troy, NY, USA}
\address[mysecondaryaddress]{Institute for Interdisciplinary Information Sciences, Tsinghua University, Beijing, China}
\address[mythirdaddress]{Department of Computer Science and Engineering, Michigan State University, 220 Trowbridge Rd, East Lansing, MI, USA}
\address[myforthaddress]{MIT-IBM Watson AI Lab, 75 Binney St, Cambridge, MA, USA}
\begin{abstract}
Motivated by the recent discovery that the interpretation maps of CNNs could easily be manipulated by adversarial attacks against network interpretability, we study the problem of interpretation robustness from a new perspective of \Renyi differential privacy (RDP). The advantages of our \ourmethod{} 
(RDP-based interpretation method) are three-folds. First, it can offer provable and certifiable top-$k$ robustness. That is, the top-$k$ important attributions of the interpretation map 
are provably robust under any input perturbation with bounded $\ell_d$-norm (for any $d\geq 1$, including $d = \infty$).  Second, our proposed method offers $\sim$10\% better experimental robustness than existing approaches in terms of the top-$k$ attributions. Remarkably, the accuracy of \ourmethod{} also outperforms existing approaches. Third, our method can provide a smooth tradeoff between robustness and computational efficiency. Experimentally, its top-$k$ attributions are {\em twice} more robust than existing approaches when the computational resources are highly constrained.
\end{abstract}
\begin{keyword}
  Differential Privacy, Machine Learning, Robustness, Interpretation, and Neural Networks
\end{keyword}
\end{frontmatter}
\section{Introduction}\label{sec:intro}

Convolutional neural networks (CNNs) have demonstrated successes on various computer vision applications, such as image classification~\cite{krizhevsky2012imagenet,he2016deep},  
object detection~\cite{ren2015faster}, semantic segmentation~\cite{gidaris2015object} and video recognition~\cite{gan2015devnet}. 
Spurred by the promising applications of CNNs, understanding why they make the correct prediction has become essential. The interpretation map explains which part of the input image plays a more important role in the prediction of CNNs. 
However, it has recently been  shown that many    interpretation maps, e.g., Simple Gradient~\cite{simonyan2013deep}, Integrated Gradient~\cite{sundararajan2017axiomatic}, DeepLIFT~\cite{shrikumar2017learning} and GradCam~\cite{selvaraju2017grad} are vulnerable to imperceptible input perturbations 
\cite{ghorbani2019interpretation,heo2019fooling}. In other words, slight perturbations to an input image could cause a significant discrepancy in its coupled interpretation map while keeping the predicted label unchanged (see an illustrative example in  Figure~\ref{fig:Our_Example}). 

\begin{figure*}[ht]
\centering
\includegraphics[width = 0.99\textwidth]{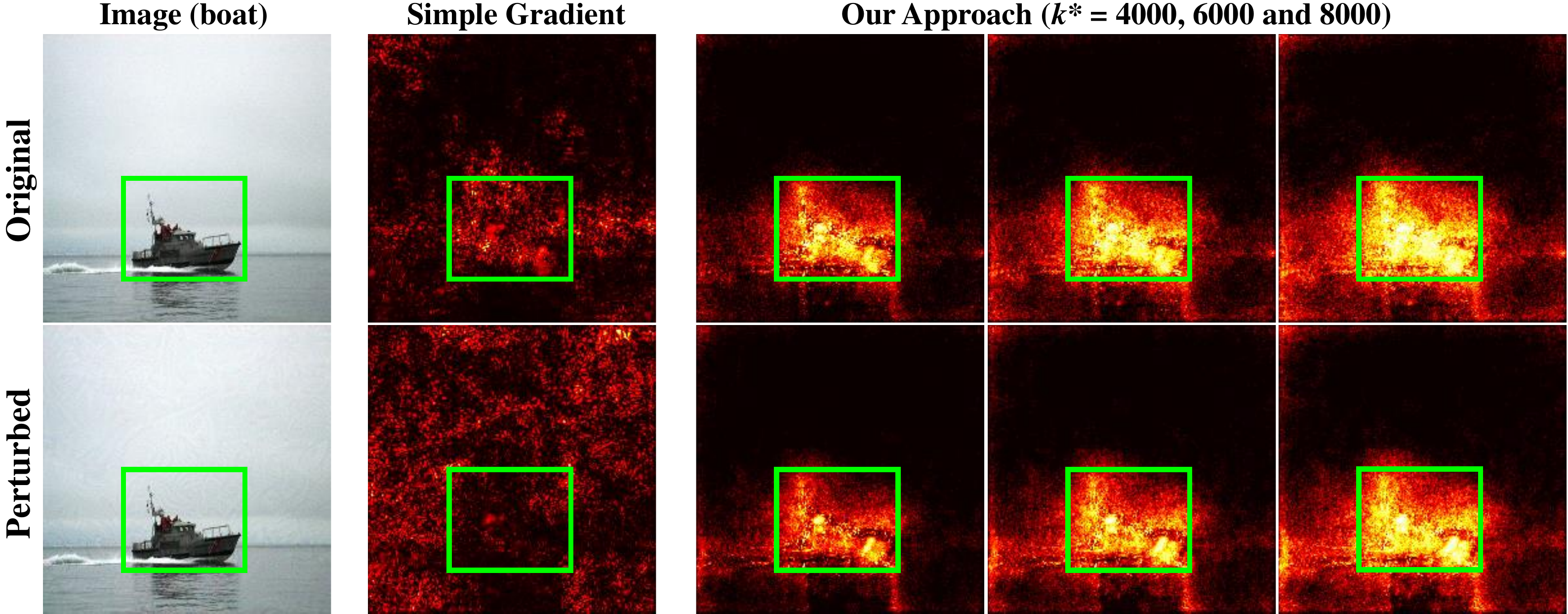}
\caption{An example of the vulnerability of {\em simple gradient}. The green boxes annotate the position of the labeled object. Our approach offers stronger robustness against interpretation attacks. Remarkably, our interpretation also more accurately matches the position of the object.
}
\vspace{-0.8em}
\label{fig:Our_Example}
\end{figure*}

{Robustness of the interpretation maps is essential.  Successful attacks can create confusions between the model interpreter and the classifier, which would further ruin the trustworthiness of systems that use the interpretations in down-stream actions, e.g., medical recommendation \cite{quellec2017deep},  
  source code captioning     \cite{ramakrishnan2020semantic},
and   transfer learning   \cite{Shafahi2020Adversarially}.}
In this paper, we aim to provide a framework to generate certifiably robust interpretation maps against $\ell_d$-norm attacks, where $d \in [1,\infty]$\footnote{In all discussions of this paper, we use $[1,\infty]$ to represent $[1,\infty)\cup\{\infty\}$.}. Our framework does not require the defender to know the type of attack, as long as it is an $\ell_d$-norm attack. Our framework provides a stronger robustness guarantee if the exact attack type is known.

\begin{figure}[ht]
\centering
\includegraphics[width = 0.6\textwidth]{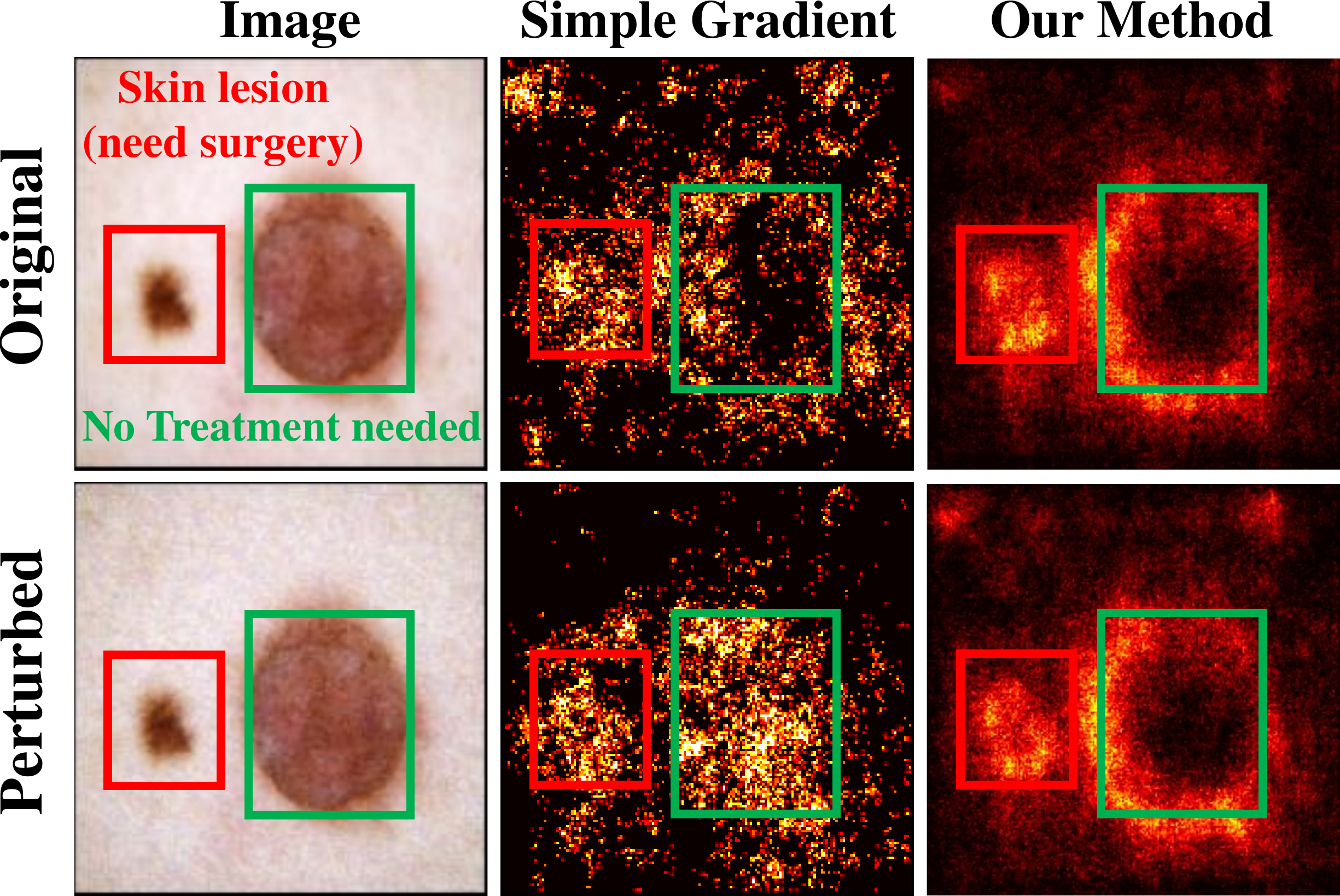}
\vspace{-0.8em}
\caption{A motivating example of interpretation robustness}
\label{fig:motivation}
\end{figure}

\noindent{\bf A Motivating Example. } To better motivate and strengthen our claim,  we present an experiment on real-world medical images: {\textit{skin lesion diagnosis}} (see Figure \ref{fig:motivation}, \cite{tschandl2018ham10000}).  In this task, the ML classifier predicts a type of skin lesion, and an interpretation map can help  doctors localize the most interpretable region responsible for the current prediction (e.g., the region highlighted by the red box  in Figure \ref{fig:motivation}). However, an adversary (that crafts imperceptible input perturbations) could fool interpretation map (for incorrect identification of region marked by the green box in Figure \ref{fig:motivation})  under the same classification label. If doctors perform a medical diagnosis based on the fooled interpretation map, then the doctor will waste his/her time on treating the incorrect region. {Figure~\ref{fig:motivation}} shows that our proposed  robust interpretation can effectively defend such an adversary.

{\em Differential privacy} (DP) has recently been introduced as a tool to improve the robustness of machine learning algorithms.   DP-based robust machine learning algorithms can provide theoretical robustness guarantees against adversarial attacks ~\cite{lecuyer2019certified}. {\em \Renyi differential privacy} (RDP) is a generalization of the standard notion of differential privacy. 
It has been proved that analyzing the robustness of CNNs using RDP can provide a stronger theoretical guarantee than standard DP~\cite{feldman2018privacy}. However, the following question remains open.

\vspace{0.5em}\textbf {How can we protect the interpretation map from $\ell_d$ attacks for $d\ge 3\;$?}

\vspace{0.5em} This is the question we will address in this paper.

\noindent{\bf Our main theoretical contribution} is the \ourmethod{} framework (Algorithm~\ref{alg:our}) that provides a smooth tradeoff between interpretation robustness and computational-efficiency. Here, the robustness of interpretation is measured by the resistance to the top-$k$ attribution\footnote{Top-$k$ attribution refers to the top-$k$ important pixels (top-$k$ largest components) in the interpretation maps.} change of an interpretation map when facing interpretation attacks. To prove the robustness of our framework, we firstly propose the new notion of \Renyi robustness (Definition~\ref{def:renyi_robust}), which is directly connected with RDP and other robustness notions. Then, we show that the interpretation robustness can be guaranteed by \Renyi robustness (Section~\ref{sec:main}). Our RDP-based analysis can provide significantly tighter robustness bound in comparison with the recently proposed interpretation method,  Sparsified SmoothGrad~\cite{levine2019certifiably} (Figure~\ref{fig:exp} left).

\noindent{\bf Experimentally,} our \ourmethod{}  
can provide $\sim$10\% improvement on  
robustness in comparison with the state of the art (Figure~\ref{fig:exp} left). Surprisingly, our improvement of interpretation robustness does not come at the cost of accuracy. We found that the accuracy of \ourmethod{}  
out-performs the recently-proposed approach of  Sparsified SmoothGrad (Figure~\ref{fig:exp} right). 
Our method works surprisingly well when computational resources are highly constrained. 
The top-$k$ attributions of our approach are twice more robust than the existing approaches when the computational resources are highly constrained (Table~\ref{tab:robust_concentration}).

\vspace{0.5em}\noindent{\bf Related Works. } Many algorithms have been proposed recently to interpret the predictions of CNNs (e.g., Simple Gradient~\cite{simonyan2013deep}, Integrated Gradient~\cite{sundararajan2017axiomatic}, DeepLIFT~\cite{shrikumar2017learning}, CAM~\cite{zhou2016learning} and GradCam~\cite{selvaraju2017grad}). The idea of adding random noise is a popular tool to improve the performance of interpretation maps. Following this idea,  SmoothGrad~\cite{smilkov2017smoothgrad}, Smooth Grad-CAM++~\cite{omeiza2019smooth}, among others, have been proposed to improve the performance of Simple Gradient, Grad-CAM, and other interpretation methods. However, none of the above-mentioned works can provide certified robustness to interpretation maps. 
We are only aware of one paper studying this problem \cite{levine2019certifiably}, whether the authors proposed an interpretation algorithm to provide certifiably robustness against only $\ell_2$-norm attack.

Adversarial training is one popular tool to improve the robustness of CNNs against adversarial attacks~\cite{sinha2017certifying,ross2018improving,tong2021facesec,li2021defending}. The method of adversarial training has also been applied to the interpretation algorithm recently~\cite{singh2019benefits}. However, there is no theoretical guarantee to the robustness of adversarial training against any norm-based attacks. Theoretical guarantees usually is also missing in other popular tools~\cite{abolghasemi2019pay,li2020robust} to improve the robustness of CNNs against adversarial attacks. 

To the best of our knowledge, we are the first to apply RDP to provide certified robustness of interpretation maps. 
The connection between DP and robustness was first revealed by~\citeauthor{dwork2009differential}~\citep{dwork2009differential}. \citeauthor{lecuyer2019certified}~\citep{lecuyer2019certified} were the first to introduce a DP-based algorithm to improve the prediction robustness of CNNs using the DP-robustness conclusion. 
RDP~\cite{mironov2017renyi} is a popular generalization of the standard notion of DP. RDP often can provide tighter privacy bound or robustness bound than standard DP in many applications~\cite{feldman2018privacy}.  The tool of RDP is also used to improve the robustness of classifiers\cite{li2019certified}. However, \citeauthor{li2019certified}~\citep{li2019certified}'s approach can only defend $\ell_1$-norm or $\ell_2$-norm attacks for classifications tasks, and usually cannot effectively defense $\ell_d$-norm attacks for $d\geq3$.

\section{Preliminaries}\label{sec:prelim}
{\bf Interpretation Attacks.} In this paper, we focus
on the interpretations of CNNs for image classification, where the input is an image $\bx$ and the output is a label in $\calC$. An interpretation of this CNN explains why the CNN makes this decision by showing the importance of the features in the classification. 
An {\em interpretation algorithm} $\bg:\,\bbR^n\times\calC \to \bbR^n$  maps the (image, label) pairs 
to an $n$-dimensional interpretation map
\footnote{In this paper, we treat both the input $\bx$ and interpretation map $\ourbm$ as vectors of length $n$. The dimension of the output can be different from that of the input. We use $\bbR^n$ to simplify the notation.}. The output of $\bg$
consists of pixel-level attribution scores, which reflect the impact of each pixel on making the prediction. Because interpretation attacks will not change in the prediction of CNNs, we sometimes omit the label part of input when the context is clear.  
In an interpretation attack, the adversary replaces $\bx$ with its perturbed version $\btx$ while keeping the predicted label unchanged. 
We assume that the perturbation is constrained by $\ell_d$-norm $||\bx-\btx||_d \triangleq \left(\sum_{i=1}^n |\bx_i-\btx_i|^d\right)^{1/d} \leq L$. When   $d = \infty$, the constraint  becomes $\max_i |\bx_i-\btx_i| \leq L$. 

\noindent{\bf Measure of Interpretation Robustness.} Interpretation maps are usually applied to identify the important features of CNNs in many tasks like object detection~\cite{chu2018deep}. Those tasks usually care more about the top-$k$ pixels of the interpretation maps. Accordingly, we adopt 
the same robustness measure as~\cite{levine2019certifiably}. Here, the robustness is measured by the overlapping ratio between the top-$k$ components of $\bg(\bx,C)$ and the top-$k$ components of $\bg(\btx,C)$. Here, we use $V_k(\bg(\bx,C),\,\bg(\btx,C))$ to denote this ratio. For example, we have $V_2((1,2,3),\,(2,1,2)) = 0.5$, because the $2^{\text{nd}}$ and $3^{\text{rd}}$ components are the top-2 components of $(1,2,3)$ while the $1^{\text{st}}$ and $3^{\text{rd}}$ components are the top-2 components of $(1,2,3)$.\footnote{We do not take the relative order between top-$k$ components into the account of this paper.} 
See Appendix~\ref{sec:overlap_def} for the formal definition of $V_k$. We define interpretation robustness as follows. 

\begin{definition}\emph{\textbf{(\bTk Robustness or Interpretation Robustness)}}\label{def:robust} For a given input $\bx$ with label $C$, we say an interpretation method $\bg(\cdot,C)$ is $\beta$-Top-$k$ robust to $\ell_d$-norm attack of size $L$ if for any $\btx$ {\em s.t.} $||\bx-\btx||_d\leq L$,
\begin{equation}\nonumber
    V_k\big(\bg(\bx,C),\bg(\btx,C)\big)\geq \beta.
\end{equation}
\end{definition}
\vspace{0.1em}

\noindent{\bf \Renyi Differential Privacy.} RDP is a novel 
generalization of standard differential privacy. RDP uses \Renyi divergence as the measure of difference between distributions. Formally,
for two distributions $P$ and $Q$ with the same support $\calS$, the \Renyi divergence of order $\alpha > 1$ is defined as:
\begin{equation}\nonumber
\begin{split}
D_{\alpha}\big(P||Q\big) &\triangleq \frac{1}{\alpha-1}\ln \bbE_{x\sim Q} \left(\frac{P(x)}{Q(x)}\right)^{\alpha}
\;\;\;\;\text{and}\;\;\;\;\;D_{\infty} = \sup_{x\in\calS} \ln\frac{P(x)}{Q(x)}.
\end{split}
\end{equation}
In this paper, we adopt a  generalized  RDP by considering ``adjacent'' inputs whose $\ell_d$ distance is no larger than $L$.\footnote{Standard RDP assumes that the $\ell_0$-norm of adjacent inputs is no more than $1$.} 

\begin{definition}[\Renyi Differential Privacy]
A randomized function $\btg$ is $(\alpha,\epsilon,L)$-\Renyi differentially private to $\ell_d$ distance, 
if for any pair of inputs $\bx$ and $\btx$ {\em s.t.} $||\bx-\btx||_d \leq L$, 
\begin{equation}\nonumber
    D_{\alpha}\big(\btg(\bx)\,||\,\btg(\btx)\big) \leq \epsilon.
\end{equation}
\end{definition}
Same as in standard DP, smaller values of  $\epsilon$ corresponds to a more private  $\btg(\cdot)$. 
RDP generalizes the standard DP, which is RDP with $\alpha = \infty$. Recall that in this paper  $\bx$ represents the input image and $\btg$ represents a (randomized) interpretation algorithm. To clarify our notations, we add~\,${\boldsymbol {\grave{}}}$~\,sign only on the top of all randomized function.

\noindent{\bf Intuitions of the RDP-Robustness Connection.} Assume the randomized function $\btg(\cdot)$ 
is \Renyi differentially private. According to the definition of RDP, for any input $\bx$, $\btg(\bx)$ (which is a distribution) is insensitive to small perturbations on $\bx$. Consider a deterministic algorithm $\bh(\bx) \triangleq \bbE_{\btg}\left[\btg(\bx)\right]$ that outputs the expectation of $\btg(\bx)$. Intuitive, $\bh(\bx)$ is also insensitive to small perturbations on $\bx$. In other words, the RDP property of $\btg(\cdot)$ leads to the robustness of $\bh(\cdot)$.

\noindent{\bf \Renyi Robustness.} Inspired by the merits of RDP, we introduce a new robustness notion: \Renyi robustness, whose merits are two-folds:\\
$(i)$ The $(\alpha,\epsilon,L)$-RDP property of $\btg(\cdot)$ directly leads to the $(\alpha,\epsilon,L)$-\Renyi robustness on $\bbE\left[\btg(\cdot)\right]$. Thus, similar with RDP, \Renyi robustness also has many desirable properties.\\
$(ii)$ \Renyi robustness is closely related to other robustness notions. If setting $\alpha = \infty$, the robustness parameter $\epsilon$ measures the average relative change of output when the input is perturbed. In Theorem~\ref{theo:renyi_btk} of Section~\ref{sec:main}, we will also prove that \Renyi robustness is closely related to \btk robustness. 

\begin{definition}[\Renyi Robustness]\label{def:renyi_robust} For a given input $\bx\in\calA$, we say a deterministic algorithm $\bh(\cdot):\,\calA\to[0,1]^{n}$ is $(\alpha,\epsilon,L)$-\Renyi robust to $\ell_d$-norm attack if for any $\bx$ and $\btx$ {\em s.t.} $||\bx-\btx||_d\leq L$,
\begin{equation}\nonumber
\begin{split}
\vD_{\alpha}(\bh(\bx)\,||\;\bh(\btx))  \triangleq\;& \frac{1}{\alpha-1}\ln\left( \prod_{i=1}^n \bh(\btx)_i\cdot\left(\frac{\bh(\bx)_i}{\bh(\btx)_i}\right)^{\alpha}\right)\\
=\;& \frac{1}{\alpha-1}\ln\left( \prod_{i=1}^n \frac{\big(\bh(\bx)_i\big)^{\alpha}}{\big(\bh(\btx)_i\big)^{\alpha-1}}\right)\\
\leq&\; \epsilon, 
\end{split}
\end{equation}
where $\bh(\cdot)_i$ refers to the $i$-th components of $\bh(\cdot)$.
\end{definition}

\begin{figure}[ht]
\centering
\includegraphics[width = 0.58\textwidth]{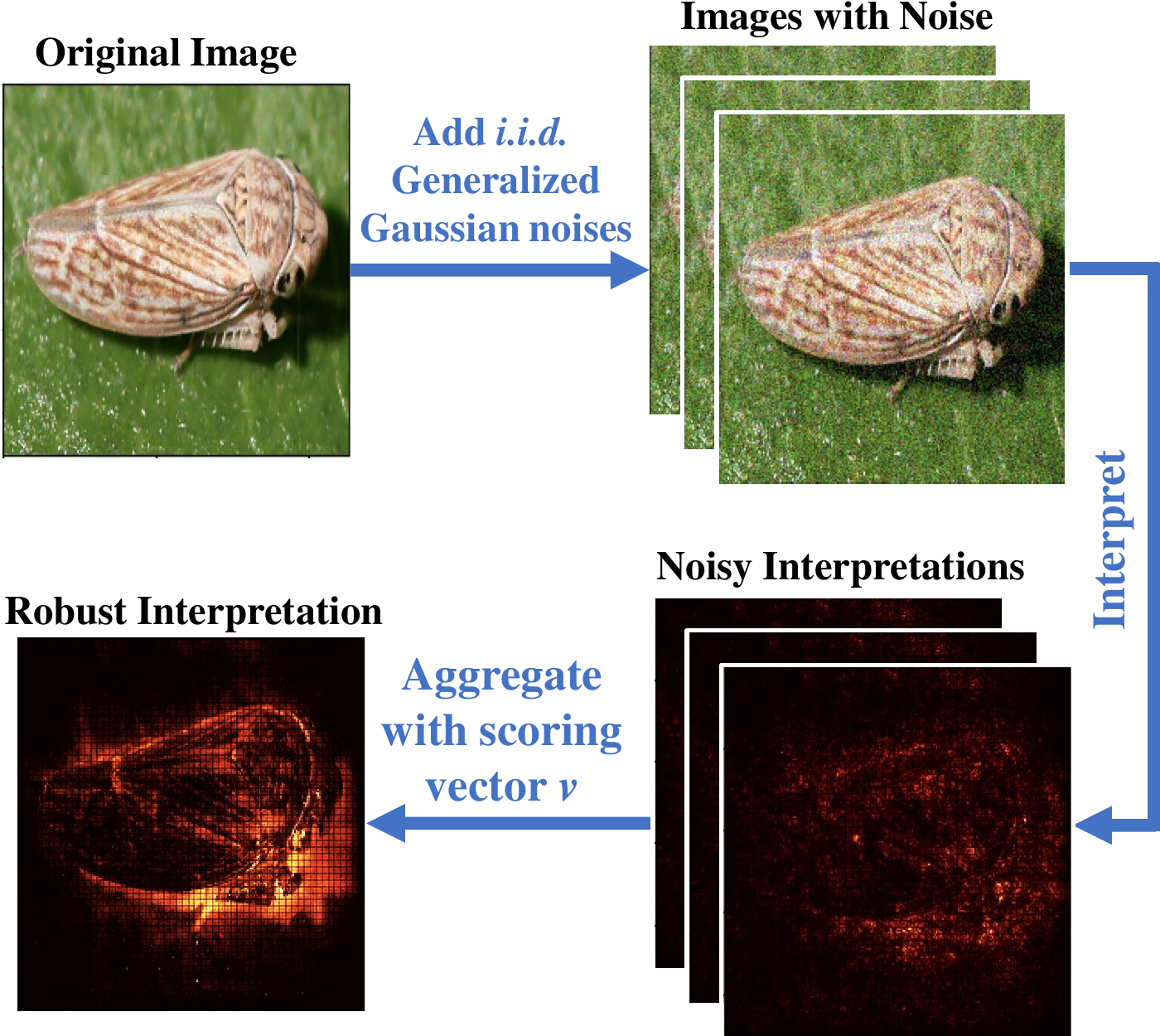}
\vspace{-0.8em}
\caption{Workflow for our \ourmethod{} with certifiable robustness against $\ell_d$-norm attack}
\label{fig:workflow}\vspace{-1.3em}
\end{figure}

\section{The \ourmethod{}  Method}\label{sec:detail_setting}
Our approach to generating robust interpretation is inspired by voting, which has been applied to improve the robustness of many algorithms with real-world applications~\cite{jenkins2000examining}. We will generate $T$ ``voters'' by introducing external noises to the input image, then use a voting rule to aggregate the outputs. Figure~\ref{fig:workflow} shows the workflow of our approach, where generalized normal noises drawn from {\em generalized normal distribution} (GND) defined in Definition~\ref{def:GND}. Here, $\Gamma(\cdot)$ refers to gamma functions. 

\begin{definition}[Generalized Normal Distribution]\label{def:GND}
A random variable $X$ follows generalized normal distribution $\calG(\mu,\sigma,b)$ if its probability density function is:
{\small
\begin{equation}\nonumber
    \frac{b\cdot\sqrt{\frac{\Gamma(3/b)}{\Gamma(1/b)}}}{2\sigma\Gamma(1/b)}\cdot\exp\left[-\left(\sqrt{\frac{\Gamma(3/b)}{\Gamma(1/b)}}\cdot\frac{|x-\mu|}{\sigma}\right)^b\right],
\end{equation}}
where $\mu$, $\sigma$ and $b$ corresponds to the expectation, standard deviation and shape factor of $X$.
\end{definition}

GND generalizes Gaussian distributions ($b = 2$) and Laplacian distribution ($b=1$). In Algorithm~\ref{alg:our}, we present our \ourmethod{} in detail. 

\begin{algorithm}
\caption{\ourmethod{} 
}\label{alg:our}
\begin{algorithmic}[1]
   \STATE {\bfseries Inputs:}  Base interpretation method $\bg$, scoring vector $\bv$, image $\bx$ and the number of samples $T$
   \STATE Generate $i.i.d.$ noises $\bdelta_1,\cdots,\bdelta_T \sim \calG(\bzero,\sigma^2I,d^*)$ 
   \STATE Calculate the noisy interpretations $\btg_t^*(\bx) \triangleq \bg(\bx+\bdelta_t)$.
   \STATE Re-scale $\btg_t^*(\bx)$ using scoring vector $\bv = (v_1,\cdots,v_n)$. The noisy interpretation after re-scaling is denoted by $\btg_t(\bx)$, where $\btg_t(\bx)_i = v_j$ if and only if $\btg_t^*(\bx)_i$ is ranked $j$-th in $\btg_t^*(\bx)$.
   \STATE{\bfseries Output} 
   Interpretation map 
   ${\ourbm} \triangleq \frac{1}{T}\sum_{t=1}^T \btg_t(\bx)$. 
\end{algorithmic}
\end{algorithm}

The shape parameter $d^*$ is set according to the prior knowledge of the defender. We assume that the defender knows the attack is based on $\ell_d$-norm where $d \leq d_{\,\text{prior}}$. Especially, $d_{\,\text{prior}}=\infty$ means the defender has no prior knowledge about the attack. Technically, we set the shape parameter $d^*$ according to the following setting:
{
\begin{equation}\nonumber
d^* = \left\{
\begin{array}{ll}
1\,,&\text{when } d_{\,\text{prior}} = 1\\
2\lceil\frac{d_{\,\text{prior}}}{2}\rceil\,,&\text{when } 1<d_{\,\text{prior}}\leq 2\lceil\frac{\ln n}{2}\rceil\\
2\lceil\frac{\ln n}{2}\rceil\,,&\text{when } d_{\,\text{prior}} > 2\lceil\frac{\ln n}{2}\rceil\\
\end{array}  
\right..
\end{equation}}
In other words, $d^*$ is the round-up of $d_{\text{prior}}$ to the next even number, except that $d=1$ or $d$ is sufficiently large. 
We set an upper threshold on $d^*$ because $\ell_{\ln n}$-norm is close to $\ell_{\infty}$-norm in practice.  W.l.o.g.~we set $v_1\geq\cdots\geq v_n$ for the scoring vector $\bv$. Our assumption guarantees that the pixels ranked higher in $\btg(\bx)$ will contribute more to our \ourmethod{} map. 


\section{\ourmethod{} with Certifiable Robustness}\label{sec:main}
In this section, we first analysis the robustness of the expected output of Algorithm~\ref{alg:our} in Section~\ref{sec:our_theory}. Then, we show the tradeoff theorem between computational efficiency and robustness in Section~\ref{sec:concentration}. The big picture of our theoretical contributions is presented in Figure~\ref{fig:proofworkflow}.  

\begin{figure}[ht]
\centering
\includegraphics[width = 0.8\textwidth]{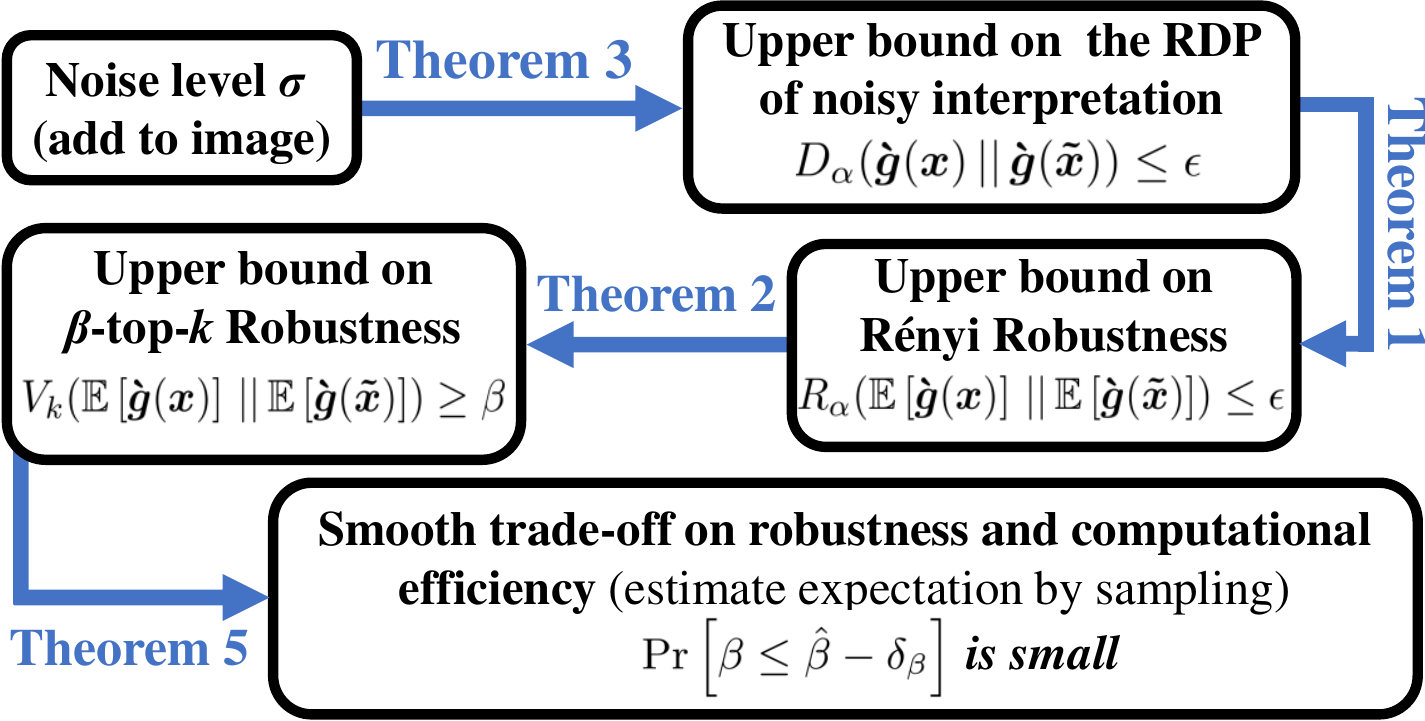}
\vspace{-0.7em}
\caption{The workflow of our proofs to the robustness of our \ourmethod.
}
\label{fig:proofworkflow}
\vspace{-1.3em}
\end{figure}



\subsection{The Expected Output of Algorithm~\ref{alg:our}}\label{sec:our_theory}

In this section, we discuss the expected output of \ourmethod, which is denoted as $\bbE\left[\btg(\bx)\right]$\footnote{In this subsection, we neglect the footnote $t$ to simplify the notation.}. As we will discuss in Theorem~\ref{theo:privacy_bound}, the noise added to image can guarantee the RDP property of $\btg(\bx)$. According to the intuitions of RDP-robustness connection shown in Section~\ref{sec:prelim}, we may expect that $\bbE\left[\btg(\bx)\right]$ is robust to input perturbations. In all discussions of this subsection, we let $\ourbm = \bbE[\btg(\bx)]$ and $\btm = \bbE\left[\btg(\btx)\right]$. 
Theorem~\ref{theo:renyi_renyi} 
connects RDP with \Renyi robustness.

\begin{theorem}[RDP$\,\to\,$\Renyi robustness]\label{theo:renyi_renyi}
If a randomized function $\btg(\cdot)$ is $(\alpha,\epsilon,L)$-RDP to $\ell_d$ distance, then $\bbE_{\btg}\left[\btg(\cdot)\right]$ is $(\alpha,\epsilon,L)$-\Renyi robust to $\ell_d$ distance.
\end{theorem}
\begin{proof}[Proof of Theorem~\ref{theo:renyi_renyi}]
Let $p_{(i,j)}$ (or $\tp_{(i,j)}$) denote the probability of the $i$-th pixel to be the $j$-th largest in the noisy interpretation (before rescaling) $\bg(\bx+\bdelta)$ (or $\bg(\btx+\bdelta)$). Then, the expectation of the $i$-th pixel in the robust interpretation map $\bbE[\ourbm_i] = \bbE_{\btg}[\btg(\bx)_i] = \sum_{j=1}^n v_j\, p_{(i,j)}$. Similarly, for the perturbed input $\btx$, we have $\bbE[\btm_i] = \bbE_{\btg}[\btg(\btx)_i] = \sum_{j=1}^n v_j\, \tp_{(i,j)}$. By the definition of \Renyi robustness, we have,
\begin{equation}\nonumber
\begin{split}
&\vD_{\alpha}(\bbE_{\btg}\left[\btg(\bx)\right]\,||\;\bbE_{\btg}\left[\btg(\btx)\right])\\
=& \vD_{\alpha}(\bbE\left[\ourbm\right]\,||\;\bbE\left[\btm\right])\\
=&  \frac{1}{\alpha-1} \ln\left(\prod_{i=1}^n \frac{\left(\sum_{j=1}^n v_j\, p_{(i,j)}\right)^{\alpha}}{\left(\sum_{j=1}^k v_j\,\tp_{(i,j)}\right)^{\alpha-1}}\right).
\end{split}
\end{equation}
Then, we apply generalized Radon's inequality~\citep{batinetu2010generalization} and have,
\begin{equation}\nonumber
\begin{split}
\vD_{\alpha}(\bbE_{\btg}\left[\btg(\bx)\right]\,||\;\bbE_{\btg}\left[\btg(\btx)\right])
\leq\frac{1}{\alpha-1} \ln\left(\sum_{j=1}^n v_j\cdot \frac{p_{(i,j)}^{\alpha}}{\tp_{(i,j)}^{\alpha-1}}\right).
\end{split}
\end{equation}
We note that the $(\alpha,\epsilon)$-RDP property of $\btg(\bx)$ provides $\frac{1}{\alpha-1} \ln\left(\sum_{j=1}^n  \frac{p_{(i,j)}^{\alpha}}{\tp_{(i,j)}^{\alpha-1}}\right)\leq \epsilon$. Using the condition of $||\bv||_1 = 1$ and $v_i\leq 1$, we have,
$$\vD_{\alpha}(\bbE_{\btg}\left[\btg(\bx)\right]\,||\;\bbE_{\btg}\left[\btg(\btx)\right])  \leq \epsilon.$$
By now, we already show the \Renyi robustness property of $\bbE_{\btg}[\btg(\cdot)]$. 
\end{proof}

In Theorem~\ref{theo:renyi_btk}, we will show that the \btk robustness property can be guaranteed by the \Renyi robustness property. To simplify our notation, we assume our \ourmethod{} map $\ourbm = (m_1,\cdots,m_n)$ is normalized ($||\ourbm||_1 = 1$). We further let $m_{i^*}$ denote the $i$-th largest component in $\ourbm$. Let $k_0 = \lfloor(1-\beta) k\rfloor+1$ to denote the minimum number of changes to violet $\beta$-ratio overlapping of top-$k$. 
Let $\calS = \{k-k_0, \cdots,k+k_0+1\}$ denote the set of last $k_0$ components in top-$k$ and the top $k_0$ components out of top-$k$.


\begin{theorem}[\Renyi robustness $\to$ \btk robustness]\label{theo:renyi_btk}
Let function $\bh(\cdot)$ to be $(\alpha,\epsilon,L)$-\Renyi robustness to $\ell_d$ distance. Then, $\ourbm = \bh(\bx)$ is $\beta$-Top-$k$ robust to $\ell_d$-norm attack of size L, if $\epsilon \leq \epsilon_{\text{\emph{robust}}}(\alpha;\beta,k;\ourbm)$, who is defined as
{\small\vspace{-0.5em}
\begin{equation}\nonumber
\begin{split}
-\ln\left( 2k_0\left(\frac{1}{2k_0}\sum_{i\in\calS}(m_{i^*})^{1-\alpha}\right)^{\frac{1}{1-\alpha}}+\sum_{i\not\in\calS}m_{i^*}\right).
\end{split}
\end{equation}}
\end{theorem}

Theorem\,\ref{theo:renyi_btk} shows the level of \Renyi robustness required by the \btk robustness condition. We provide some insights on $\epsilon_{\text{robust}}$. 
There are two terms inside of the $\ln(\cdot)$ function. The first term 
corresponds to the minimum information loss to replace $k_0$ items from the top-$k$. The second term corresponds to the unchanged components in the process of replacement. The proof of Theorem~\ref{theo:renyi_btk} can be found in~\ref{app:renyi_btk}.

Combining the results in Theorem~\ref{theo:renyi_renyi} and Theorem~\ref{theo:renyi_btk}, we know that \btk robustness can be guaranteed by the \Renyi differential privacy property on the randomized interpretation algorithm $\btg(\bx) = f_{\bv}(\bg(\bx+\bdelta,C))$, where $f_{\bv}$ is the re-scaling function using scaling vector $\bv$. Next, we show the RDP property of $\btg(\cdot)$. In the remainder of this paper,  we let $\Gamma(\cdot)$ to denote the gamma function and let $\mathbbm{1}(\cdot)$ to denote the indicator function. To simplify notations, we let $\epsilon_{\alpha}\Big(\frac{L}{\sigma}\Big) = \frac{1}{\alpha-1}{\ln\left[\frac{\alpha}{\alpha-1}\exp\Big(\frac{(\alpha-1)L)}{\sqrt 2 \sigma}\Big)+\frac{\alpha-1}{2\alpha-1}\exp\Big(\frac{-\alpha L}{\sqrt 2 \sigma}\Big)\right]}$ and  $\epsilon_{d^*}\Big(\frac{L}{\sigma}\Big)= \frac{1}{\Gamma(1/{d^*})}\sum_{i=1}^{{d^*}/2}\binom{2i}{{d^*}}\left(\frac{L}{\sigma^*}\right)^{2i}\Gamma\Big(\frac{{d^*}+1-2i}{{d^*}}\Big),$ where $\sigma^* = \sqrt{\frac{\Gamma(1/b)}{\Gamma(3/b)}} \,\sigma$.

\begin{theorem}[Noise level$\,\to\,$RDP]\label{theo:privacy_bound}
For any re-scaling function $f_{\bv}(\cdot)$, let $\btg(\bx) = f_{\bv}(\bg(\bx+\bdelta))$ 
where $\bdelta\sim\calG(\bzero,\sigma^2 I,d^*)$. Then, $\btg$ has the following properties with $\ell_{d}$ distance:\\
$(i)$ $\left(1,\,\epsilon_{d^*}\big(\frac{L}{\sigma}\cdot\exp\left[\mathbbm{1}(d>2\lceil\frac{\ln n}{2}\rceil)\right]\big),L\right)$-RDP for all $d\geq 2$.\\
$(ii)$ For all $\alpha\geq 1$, we have $\left(\alpha,\frac{\alpha L^2}{2\sigma^2},L\right)$-RDP when ${d} \in(1,2] \;$ and $\;\left(\alpha,\epsilon_{\alpha}\big(\frac{L}{\sigma}\big),L\right)$-RDP when ${d} = 1$.
\end{theorem}
\begin{proof}
According to the {\em post-processing} property of RDP~\citep{mironov2017renyi}, we know that 
$D_{\alpha}\big(\btg(\bx)\,||\; \btg(\btx)\big) \leq D_{\alpha}(\bx+\bdelta\,||\,\btx+\bdelta).$
When $d\in(1,2]$, we always have $||\bx-\btx||_{d}\geq ||\bx-\btx||_{2}$. Thus, all conclusion requiring $||\bx-\btx||_{d}\leq L$ will also hold when the requirement becomes $||\bx-\btx||_{2}\leq L$. Then, the $d\in(1,2]$ case of $(ii)$ in Theorem~\ref{theo:privacy_bound} follows by the standard conclusion of RDP on Gaussian mechanism and the $d=1$ case follows by the standard conclusion of RDP on Laplacian mechanisms. In Lemma~\ref{lem:upper}, we prove the RDP bound for generalized normal mechanisms of even shape parameters. This bound can be directly applied to bound the KL-privacy of generalized normal mechanisms~\cite{wang2016average}. 

\begin{lemma}\label{lem:upper}
For any positive even shape factor $b$, letting $x\sim\calG(0,\sigma,b)$ and $\tilde x\sim \calG(L,\sigma,b)$,  we have
$$\lim_{\alpha\to1^{+}} D_{\alpha}(x\,||\,\tilde x)
= \epsilon_{d^*}\big(L/\sigma\big).$$
\end{lemma}
\begin{proof}[Proof of Lemma~\ref{lem:upper}]
By the definition of \Renyi divergence, we have,
{\small
\begin{equation}\nonumber
\begin{split}
& \lim_{\alpha\to 1} D_{\alpha}(x||x') = \lim_{\alpha\to 1} \left( \frac{1}{\alpha-1}\cdot\ln\int_{-\infty}^{\infty}\frac{b\cdot\exp\left(-(\frac{x-\mu}{\sigma^*})^b\right)}{2\sigma^*\Gamma(1/b)}\frac{\exp\left(-\alpha(\frac{x}{\sigma^*})^b\right)}{\exp\left(-\alpha(\frac{x-\mu}{\sigma^*})^b\right)}\dx\right).
\end{split}
\end{equation}
}
Because we interested in the behaviour of $D_{\alpha}$ when $\alpha\to 1$, to simplify notation, we let $\delta = \alpha-1$. when $\alpha-1 \to 0$, we apply first-order approximation to $\exp(\cdot)$ and have,
{\small
\begin{equation}\nonumber
\begin{split}
& \lim_{\alpha\to 1} D_{\alpha}(x||x') = \lim_{\delta\to 0^{+}} \left( -\delta^{-1}\cdot\ln\int_{-\infty}^{\infty}\frac{b\cdot\exp\left(-(\frac{x}{\sigma^*})^b\right)}{2\sigma^*\Gamma(1/b)}\left(1-\delta(\frac{x}{\sigma^*})^b+\delta(\frac{x-\mu}{\sigma^*})^b\right)\dx\right).
\end{split}
\end{equation}
}
Because $\frac{b\cdot\exp\left(-(\frac{x}{\sigma^*})^b\right)}{2\sigma^*\Gamma(1/b)}$ is the PDF of $\calG(0,\sigma,b)$, we have,
{\small
\begin{equation}\nonumber
\begin{split}
& \lim_{\alpha\to 1} D_{\alpha}(x||x') = \lim_{\delta\to 0^{+}} \left( -\delta^{-1}\cdot\ln\left[1+\delta\int_{-\infty}^{\infty}\frac{b\cdot\exp\left(-(\frac{x}{\sigma^*})^b\right)}{2\sigma^*\Gamma(1/b)}\left(-(\frac{x}{\sigma^*})^b+(\frac{x-\mu}{\sigma^*})^b\right)\dx\right]\right).
\end{split}
\end{equation}
}
By applying first-order approximation to $\ln(\cdot)$, we have, 
{
\begin{equation}\nonumber
\begin{split}
& \lim_{\alpha\to 1} D_{\alpha}(x||x')=\int_{-\infty}^{\infty}\frac{b\cdot\exp\left(-(\frac{x}{\sigma^*})^b\right)}{2\sigma^*\Gamma(1/b)}\left((\frac{x-\mu}{\sigma^*})^b-(\frac{x}{\sigma^*})^b\right)\dx.
\end{split}
\end{equation}
}
Then, we expand $\left(\frac{x-\mu}{\sigma^*}\right)^b$ and have,
{
\begin{equation}\nonumber
\begin{split}
& \lim_{\alpha\to 1} D_{\alpha}(x||x')=\int_{-\infty}^{\infty}\frac{b\cdot\exp\left(-(\frac{x}{\sigma^*})^b\right)}{2\sigma^*\Gamma(1/b)}\sum_{i=1}^{b}\binom{i}{b}\left(\frac{\mu}{\sigma^*}\right)^{i}\left(\frac{x}{\sigma^*}\right)^{b-i}\dx.
\end{split}
\end{equation}
}
When $b-i$ is odd, $\exp\left(-(\frac{x}{\sigma^*})^b\right)\left(\frac{x}{\sigma^*}\right)^{b-i}$ is an odd function and the integral will be zeros. When $b-i$ is even, it will become an even function. Thus, we have,
{
\begin{equation}\nonumber
\begin{split}
& \lim_{\alpha\to 1} D_{\alpha}(x||x')=\int_{0}^{\infty}\frac{b\cdot\exp\left(-(\frac{x}{\sigma^*})^b\right)}{\sigma^*\Gamma(1/b)}\sum_{i=1}^{b/2}\binom{2i}{b}\left(\frac{\mu}{\sigma^*}\right)^{2i}\left(\frac{x}{\sigma^*}\right)^{b-2i}\dx.
\end{split}
\end{equation}
}
Through substituting $\left(\frac{x}{\sigma^*}\right)^b$, we have,
{
\begin{equation}\nonumber
\begin{split}
\lim_{\alpha\to 1} D_{\alpha}(x||x')
=&\int_{0}^{\infty}\frac{\exp\left(-y\right)}{\Gamma(1/b)}\sum_{i=1}^{b/2}\binom{2i}{b}\left(\frac{\mu}{\sigma^*}\right)^{2i}y^{(1-2i)/b}\dy.
\end{split}
\end{equation}
}
Finally, Lemma~\ref{lem:upper} follows by the definition of Gamma function.
\end{proof}
When $d\leq 2\lceil\frac{\ln n}{2}\rceil$, we always have $d^*\geq d$ and $(i)$ of this case holds according to the same reason as $(ii)$. When $d > 2\lceil\frac{\ln n}{2}\rceil$, we have $||\bx-\btx||_{d^*} \leq n^{1/d^*-1/d}\cdot||\bx-\btx||_d \leq e\cdot||\bx-\btx||_d$.  
Thus, $(i)$ of Theorem~\ref{theo:privacy_bound} also holds when $d > 2\lceil\frac{\ln n}{2}\rceil$.
\end{proof}


Combining the conclusions of Theorem~\ref{theo:renyi_renyi}-\ref{theo:privacy_bound}, we get the theoretical robustness of \ourmethod{} in Table~\ref{tab:result}. In the table, $\epsilon^{-1}(\cdot)$ denotes the inverse function of $\epsilon(\cdot)$.

\begin{table*}[ht]
\small
\centering
\renewcommand\arraystretch{1.35}
\resizebox{1.00\textwidth}{!}{
\begin{tabular}{|p{2.28cm}|p{3.19cm}|p{8.45cm}|}
\hline
Prior knowledge & Distribution of noise & Maximum attack size $L$ without violating \btk robustness\\
\hline 
$d_{\,\text{prior}} = 1$    & Laplacian Distribution & $L = \sigma \cdot \sup_{\alpha > 1}\epsilon_{\alpha}^{-1}(\epsilon_{\text{robust}}(\alpha))$\\
\hline
$d_{\,\text{prior}}\in(1,2]$     & Gaussian Distribution & $L = \sigma\cdot\sup_{\alpha > 1}\sqrt{2\epsilon_{\text{robust}}(\alpha)/\alpha}$ \\
\hline
$d_{\,\text{prior}}\in(2,\infty]$ & GND with $b = d^*$    & $L = \sigma\cdot \exp\left(-\mathbbm{1}(d>2\lceil\frac{\ln n}{2}\rceil)\right) \cdot \epsilon_{d^*}^{-1}\left(\lim_{\alpha\to 1^+}\epsilon_{\text{robust}}(\alpha)\right)$\\
\hline
\end{tabular}}
\vspace{-0.7em}
\caption{Theoretical \btk robustness for $\ourbm = \bbE\left[\btg(\bx)\right]$ when the defender has different prior knowledge on the attack types.}\label{tab:result}
\vspace{-1.2em}
\end{table*}

\subsection{The Smooth Trade-off between Robustness and Computational Efficiency}\label{sec:concentration}

In all analysis in Section~\ref{sec:our_theory}, we assumed that the value of $\ourbm = \bbE[\btg(\bx)]$ in Algorithm~\ref{alg:our} can be computed efficiently. 
However, there may not even exist a closed form expression of $\ourbm$ and thus $\ourbm$ may cannot not be computed efficiently. In this section, we study the robustness-computational efficiency trade-off when $\ourbm$ is approximated through sampling. That is, approximate $\ourbm$ using $\sum_{t=1}^T \btg_t(\bx)$ (the same procedure as Algorithm~\ref{alg:our}). To simplify notation, we let $\hat \beta$ to denote the calculated top-$k$ robustness from Table~\ref{tab:result}. We use $\beta$ to denote the real robustness parameter of $\ourbm$. We note that our approach will become more computational-efficient when $T$ becomes smaller. In Theorem~\ref{theo:concentration}, we show that the \Renyi robustness parameter $\epsilon_{\text{robust}}$ will have a larger probability to be larger when the number of sample $T$ becomes larger. The conclusion on attack size $L$ or robust parameter $\beta$ follows by applying Theorem~\ref{theo:concentration} to Table~\ref{tab:result} or Theorem~\ref{theo:renyi_btk} respectively. The formal version of Theorem~\ref{theo:concentration} can be found in~\ref{sec:app:concen}.

\begin{theorem}[Smooth Trade-off Theorem, Informal]~\label{theo:concentration}
Letting $\hat{\epsilon}_{\text{\emph{robust}}}$ to denote the estimated \Renyi robustness parameter from $T$ samples, we have
{
\begin{equation}\nonumber
\Pr\left[{\epsilon}_{\text{\emph{robust}}} \geq (1-\delta_{\epsilon})\hat{\epsilon}_{\text{\emph{robust}}}\right]
\geq 1- \text{negl}\left(T\right),\end{equation}} 
where negl$(\cdot)$ refer to a negligible function.
\end{theorem}


\section{Experimental Results}\label{sec:exp}
In this section, we experimentally evaluate the robustness and the interpretation accuracy of \ourmethod. We show that our approach performs better on both ends against the existing approach of Sparsified SmoothGrad~\cite{levine2019certifiably}.

\noindent{\bf General setups of our implementation.} We use PASCAL Visual Object Classes Challenge 2007 dataset  (VOC2007,~\cite{pascal-voc-2007}) to evaluate both interpretation robustness and interpretation accuracy. This enables us to compare the annotated object positions in VOC2007 with the interpretation maps to benchmark the accuracy of interpretation methods.  We adopt VGG-16~\cite{simonyan2014very} as the CNN backbone for all the experiments.
Simple Gradient~\cite{simonyan2013deep} is used as the base interpretation algorithm\footnote{We note that the base interpretation algorithm is not our baseline. The baseline throughout this paper is Sparsified SmoothGrad.} (denoted as $\bg(\cdot)$ in the input of Algorithm~\ref{alg:our}).

\noindent{\bf Defense and attack configurations.} We focus our study on the most challenging case of $\ell_{d}$-norm attack: $\ell_{\infty}$-norm attack.
We examine the robustness and accuracy of our methods under the standard $\ell_{\infty}$-norm attack against the top-$k$ component of Simple Gradient, which is firstly introduced by~\cite{ghorbani2019interpretation}. Formally, our attack method is presented in in algorithm~\ref{alg:attack}. Here, $B$ denotes the set of top-$k$ components' subscripts.  To be more specific, we set the size of attack $L = 8/256 \approx 0.03$, learning rate $lr = 0.5$ and the number of iteration $T = 300$.
\begin{algorithm}
\caption{$\ell_\infty$-norm Attack on Top-$k$ Overlap~\citep{ghorbani2019interpretation}
}\label{alg:attack}
\begin{algorithmic}
   \STATE {\bfseries Inputs:}  An integer $k$, learning rate $lr$, an input image $\bx \in \bbR^n$; a interpretation method $\bg(\cdot,\,\cdot)$, maximum $\ell_{\infty}$-norm perturbation $L$, and the number of iterations $T$
   \STATE{\bfseries Output:} Adversarial example $\btx$.
   \STATE Define $D(\bz) = -\sum_{i\in B} \bg(\bx)_i$.
   \STATE {\bfseries Initialization:} $\bx^0 = \bx$.
   \FOR{$t=1$ {\bfseries to} $T$}
   \STATE $\bx^t \gets \bx^{t-1} + lr\cdot\frac{\nabla D(\bx^{t-1})}{\gamma}$
   \IF{$||\bx^t-\bx||_d > \rho$}
   \STATE $\bx^t \gets \bx + \rho\cdot\frac{\bx^t-\bx}{||\bx^t-\bx||_d}$
   \ENDIF
   \ENDFOR
   \STATE{\bfseries Output:} $\btx = \argmax_{\bz\in\{\bx^0,\cdots,\bx^T\}} D(\bz)$.
\end{algorithmic}
\end{algorithm}

According to our noise setup discussed in Section~\ref{sec:detail_setting}, we set the shape factor of GND as $b = 10$ in consideration of VOC2007 dataset image size. To compare with other approaches, the standard deviation of noise is fixed to be 0.1. The scoring vector $\bv$ used to aggregate ``votes'' is designed according to a sigmoid function. We take $v_i = \frac{1}{Z}\cdot\left[1+e^{\eta\cdot(i-k^*)}\right]^{-1}$, where $Z = \sum_{i'=1}^n\left[1+e^{\eta\cdot(i'-k^*)}\right]^{-1}$ is the normalization factor. $k^*$ and $\eta$ are user-defined parameters to control the shape of $\bv$. In all discussions of Section~\ref{sec:exp}, we set $\eta = 10^{-4}$. See Figure~\ref{fig:Our_Example} in Section~\ref{sec:intro} for an illustration of our interpretation maps.


\subsection{Robustness of \ourmethod}\label{sec:robust}
We compare the \btk robustness of \ourmethod{} with Sparsified SmoothGrad. \ourmethod{} gets not only tighter robustness bound, but stronger robustness against $\ell_{\infty}$-norm attack in comparison with baseline.  For both  Sparsified SmoothGrad and our approach, we set the number of samples $T$ to be $50$.

\noindent{\bf Experimental Robustness.}  Figure~
\ref{fig:exp}
shows the \btk robustness of our approach in comparison with  Sparsified SmoothGrad. Here, we adopt the transfer-attack setting, where an attack to the top-$2500$ of Simple Gradient is applied to both our approach and   Sparsified SmoothGrad. We plot the top-$k$ overlapping ratio $\beta$ between the interpretation of the original image $\bx$ and the adversarial example $\btx$. Our approach has consistently better robustness ($\sim 10\%$) than Sparsified SmoothGrad under different settings on $k$. Also see Table~\ref{tab:robust_concentration} for the experimental robustness comparison when the parameters $T$ is different. 

\begin{figure}[ht]
\centering     
\includegraphics[width=0.7\textwidth]{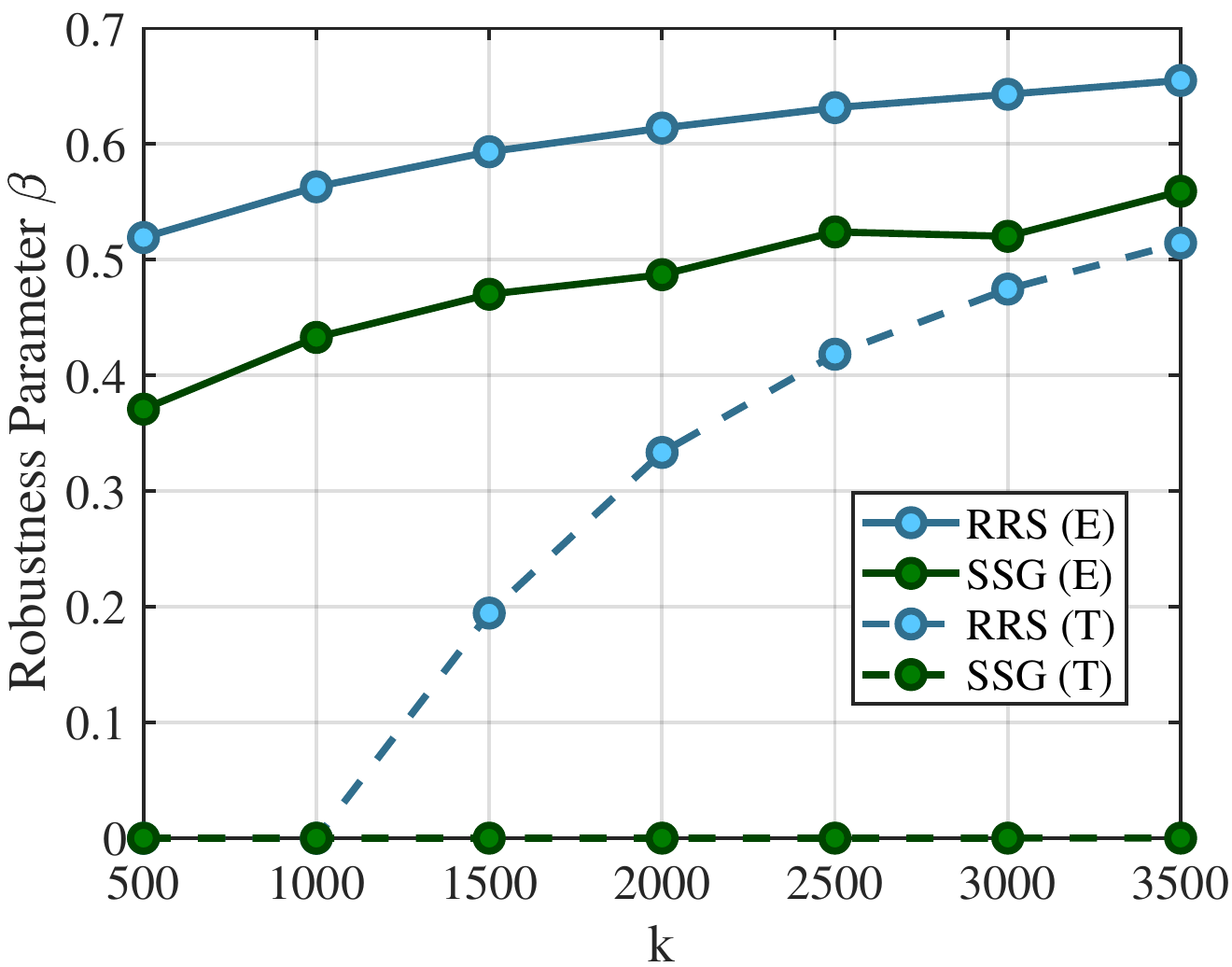}
\vspace{-0.65em}
\caption{ {
The \btk robustness of \ourmethod{} (RRS) in comparison with Sparsified SmoothGrad (SSG). (E) and (T) corresponds to the experimental and theoretical robustness respectively. 
}}\label{fig:exp}\vspace{-0.8em}
\end{figure}

\noindent{\bf Theoretical Robustness Bound.} The dash lines in Figure~
\ref{fig:exp} presents the theoretical robustness bound under $\ell_{\infty}$-norm attack. One can see that our robustness bound is much tighter than the bound provided in~\cite{levine2019certifiably}. This because the tool used by~\cite{levine2019certifiably} is hard to be generalized to $\ell_{\infty}$-norm attack, and a direct transfer from $\ell_2$-norm to $\ell_{\infty}$-norm will result in an extra loss of $\Theta(\sqrt{n})$. 

\noindent We use the following experiments to further illustrate our theory under different attack sizes and different noise levels (the standard deviation of generalized normal distribution). The experiment setting is the same as Figure~\ref{fig:exp}. Our results are summarized in the two charts below, where $\beta$ is the robustness parameter in Definition~\ref{def:robust} (larger means more robust).

\begin{table}[ht]
\centering
\begin{tabular}{cccccc}
\toprule
Noise level $\sigma$& 0.07& 0.1 &0.15& 0.2& 0.3\\\midrule
Experimental $\beta$ &0.535& 0.665& 0.720& 0.745& 0.748\\         
Theoretical $\beta$ &0.169& 0.480& 0.652& 0.707& 0.729
\\ \bottomrule
\end{tabular}
\vspace{-0.5em}
\caption{Experimental and theoretical top-$k$ robustness under different noise level} \label{tab:robust1}
\end{table}

\begin{table}[ht]
\centering
\begin{tabular}{cccc}
\toprule
Attack size $L$& 0.02& 0.03 &0.04\\\midrule
Experimental $\beta$ &0.674& 0.656& 0.628\\         
Theoretical $\beta$ &0.672& 0.508& 0.271
\\ \bottomrule
\end{tabular}
\vspace{-0.5em}
\caption{Experimental and theoretical top-$k$ robustness under different size of attack $L$} \label{tab:robust2}
\end{table}

\subsection{Accuracy of Interpretations}
Accuracy is another main concern of interpretation maps. That is, to what extent the main attributes of interpretation overlap with the annotated objects. In this section, we introduce a generalization of {\em pointing game}~\cite{zhang2018top,fong2019understanding} to evaluate the performance of our interpretation method. In a pointing game, an interpretation method calculates the interpretation map and compare it with the annotated object. If the top pixel in the interpretation map is within the object, a score ``$+1$'' will be to the object. Otherwise, a score ``$-1$'' will be granted. The pointing game score is the average score of all objects.

\begin{figure}[ht]
\centering
\includegraphics[width = 0.65\textwidth]{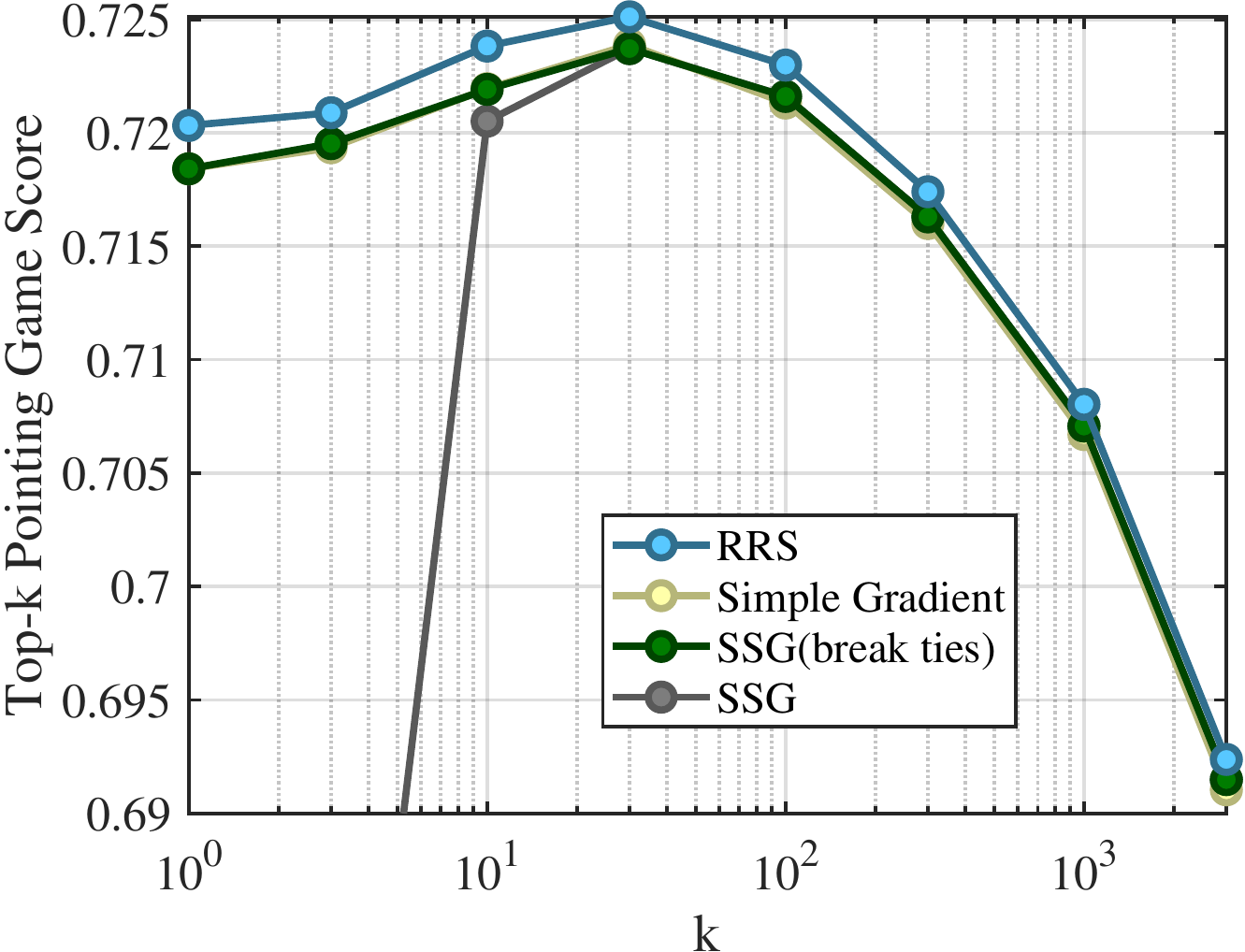}
\vspace{-0.8em}
\caption{{
The accuracy of \ourmethod{} (RRS) in comparision with Sparsified SmoothGrad (SSG) and Simple Gradient. The accuracy of SSG for $k<10$ is too small and was ignored to improve the presentation. 
}}
\label{fig:accuracy}\vspace{-0.8em}
\end{figure}

However, not only the top-$1$ pixel effect the quality of interpretation maps. Hence, we generalize the idea of pointing game to the top-$k$ pixels by checking the ratio of top-$k$ pixels of interpretation map within the region of the object and applying the same procedure as the standard pointing game.
Figure~
\ref{fig:accuracy} shows the result of top-$k$ pointing game on various interpretation maps. The annotated object position in VOC2007 dataset ({\em e.g.}, the green boxes in Figure~\ref{fig:Our_Example}) is taken as the ground-truth. The parameters of  Sparsified SmoothGrad and our approach are both the same as the settings in Section~\ref{sec:robust}. In figure~\ref{fig:accuracy}, we compare \ourmethod{} with two versions of Sparsified SmoothedGad. The first version (SSG) follows the same setting as~\cite{levine2019certifiably}, which randomly break the ties when calculate the its top-$k$ component. For the second version (SSG (break ties)), we broke the ties according to descending order for the summation of all noisy interpretations. This order used for tie-breaking is the same as the order for standard SmoothedGrad. To our surprise, our approach's accuracy also turns out to be better than  Sparsified SmoothGrad in both settings, even though robustness is usually considered an opposite property of accuracy. 

\subsection{Computational Efficiency-Robustness Tradeoff}\label{sec:real_concentration}
Computational efficiency is another main concern of interpretation algorithms. However,
most of previous works on interpretation of certified robustness did not pay much attention on the  computational efficiency.
If applying the settings of those works to larger images ({\em i.e.,} ImageNet or VOC figures) or complex CNNs ({\em i.e.,} VGG or GoogleLeNet), it may take hours to generate a single interpretation map. Thus, we experimentally verify our approach's performance when the number of generated noisy interpretations $T$ is no larger than $30$. Here,  $T$ can be used as a measure of computational resources because the time to compute on a noisy interpretation map is similar in our approach and  Sparsified SmoothGrad. Also note that the time taken by other steps are negligible in comparison with computing the noisy interpretations. Table~\ref{tab:robust_concentration} shows the top-$k$ robustness of our approach and   Sparsified SmoothGrad  when $T$ is small. We observe that the robustness of both algorithms will decrease when $T$ becomes smaller. However, our approach's robustness decreases much slower than  Sparsified SmoothGrad when the computational resources become more limited. When $T=5$, our \ourmethod{} becomes {\em more than twice} more robust than  Sparsified SmoothGrad.

\begin{table}[ht]
\centering
\resizebox{0.95\textwidth}{!}{
\begin{tabular}{ccccccc}
\toprule
\#Sample $T$ &   5 &    10 & 15 & 20 & 25 & 30\\\midrule
$\beta$ for RRS &  {\bf 50.84\%} &  {\bf 58.02\%}  & {\bf 60.80\%} &     {\bf 63.60\%} & {\bf 64.48\%} & {\bf 65.67\%}\\         
$\beta$ for SSG &    22.10\% & 41.90\%    & 47.67\% &    53.93\% & 54.99\% & 56.99\%
\\ \bottomrule
\end{tabular}}
\vspace{-0.8em}
\caption{\btk robustness when the computational resources is highly constrained.} \label{tab:robust_concentration}\vspace{-1.3em}
\end{table}

\section{Conclusions}
In this paper, we firstly build a bridge to connect the property of RDP with the robustness of interpretation maps. Based on that, we propose a simple yet effective method to generate interpretation maps with certifiable robustness against broad kinds of attacks. Our approach can prevent the model interpreter and the model classifier from being confused by interpretation attacks. Our theoretical guarantee on robustness can provide tighter bounds than existing works against any $\ell_{d}$-norm attacks for all $d\in[1,\infty]$. 
Finally, we experimentally show that our approach can provide both better robustness and better accuracy than the recently proposed approach of Sparsified SmoothGrad.

\section{Acknowledgement}
Ao Liu acknowledges the IBM AIHN scholarship for support. Lirong Xia acknowledges NSF \#1453542, NSF \#1716333, and ONR \#N00014-17-1-2621 for support. 

{

}

\clearpage
\onecolumn
{\noindent\bf This is the supplementary material of  \\
Certifiably Robust Interpretation via \Renyi Differential Privacy}

\appendix

\section{Missing Definitions}\label{sec:overlap_def}
In this section, we give a formal definition to the top-$k$ overlapping ratio. Before proceeding, we first define the set of top-$k$ component $T_k(\cdot)$ of a vector. Formally, 
$$T_k(\bx) \triangleq \left\{x:\;x\in\bx\;\wedge\;\#\{x':\;x'\in\bx\;\wedge\;x'\geq x\}\leq k\right\}$$
Now, we are ready to formally define top-$k$ overlapping ratio:
\begin{definition}[Top-$k$ Overlap]
Using the notations above, the top-$k$ overlap ratio $V_k$ between any pair of vectors $\bx$ and $\btx$ is defined as:
$$V_k(\bx,\btx)\triangleq \frac{1}{T}\left(\#\big[T_k(\bx)\cap T_k(\btx)\big]\right).$$
\end{definition}



\section{Proof of Theorem~\ref{theo:renyi_btk}}\label{app:renyi_btk}
\subsection{A readable proof for Theorem~\ref{theo:renyi_btk}}
\begin{proof}[Proof of Theorem~\ref{theo:renyi_btk}]
Mathematically, we calculate the minimum change in \Renyi divergence to violet the requirement of \btk robustness. That is, calculate
$$\min_{\bbE[\btm]\;s.t.\;V_k(\bbE[\ourbm],\bbE[\btm]) \leq \beta}\vD_{\alpha}(\bbE[\ourbm]\,|| \;\bbE[\btm]).$$
In all remaining proof of Theorem~\ref{theo:renyi_btk} we slightly abuse the notation and let $m_i$ and $\tm_i$ to represent $\bbE[m_i]$ and $\bbE[\tm_i]$ respectively. W.L.O.G., we assume $m_1\geq\cdots\geq m_n$. 
Then, we show that we must have $\tm_1\geq\cdots\geq \tm_{k-k_0-1}\geq \tm_{k-k_0} = \cdots = \tm_{k+k_0+1} \geq \tm_{k+k_0+2} \geq \cdots \geq \tm_n$ to reach the minimum of \Renyi divergence. 
To simplify notation, we let $s(\bbE[\ourbm]\,|| \;\bbE[\btm]) = \sum_{i=1}^n m_i\left(\frac{\tm_i}{m_i}\right)^{\alpha}$. One can see that $s(\bbE[\ourbm]\,|| \;\bbE[\btm])$ reaches the minimum on the same condition as $\vD_{\alpha}(\bbE[\ourbm]\,|| \;\bbE[\btm])$. To outline our proof, we prove the following claims one by one (See section~\ref{sec:missing} for the detailed proofs).

$(i)$ [Natural] To reach the minimum, there are exactly $k_0$ different components in the top-$k$ of $\bbE[\ourbm]$ and $\bbE[\btm]$.

$(ii)$ To reach the minimum, $\tm_{k-k_0},\cdots,\tm_k$ are not in the top-$k$ of $\bbE[\btm]$.

$(ii')$ To reach the minimum, $\tm_{k+1},\cdots,\tm_{k+k_0+1}$ must appear in the top-$k$ of $\bbE[\btm]$.

$(iii')$ [Proved in~\cite{li2019certified}] To reach the minimum, we must have $\tm_i\geq \tm_j$ for any $i\leq j$.

One can see that the above claims on $\bbE[\btm]$ is equivalent to the following KKT condition:
\begin{equation}\nonumber
\begin{split}
&\min_{\tm_1\cdots,\tm_{n}} \sum_{i=1}^n m_i\left(\frac{\tm_i}{m_i}\right)^{\alpha}\\
\text{subject to }&\;\;\sum_{i=1}^n \tm_i = 1\\
\text{subject to }&\;\;\tm_j - \tm_i \leq 0,\;\forall i<k\\
\text{subject to }&\;\;-\tm_i \leq 0,\;\forall i\in[k]\\
\text{subject to }&\;\;\tm_j-\tm_i = 0,\;\forall i,j\in\calS\\
\end{split}
\end{equation}
Solving it, we know $s(\bbE[\ourbm]\,|| \;\bbE[\btm])$ reaches minimum when
\begin{equation}\nonumber
\begin{split}
\forall i\in\calS,\,\tm_i = \frac{\breve{m}}{2k_0\breve{m}+\sum_{i\not\in\calS}m_i}\;\;\;\;\text{and}\;\;\;\;
\forall i\not\in\calS,\,\tm_i = \frac{m_i}{2k_0\breve{m}+\sum_{i\not\in\calS}m_i},\\
\end{split}
\end{equation}
where $\breve{m} = \left(\frac{1}{2k_0}\sum_{i\in\calS}(m_i)^{1-\alpha}\right)^{\frac{1}{1-\alpha}}$. By plugging in the above condition, we have,
\begin{equation}\nonumber
\begin{split}
&\min_{\bbE[\btm]\;s.t.\;V_k(\bbE[\ourbm],\bbE[\btm]) \leq \beta}\vD_{\alpha}(\bbE[\ourbm]\,|| \;\bbE[\btm])
=-\ln\left( 2k_0\left(\frac{1}{2k_0}\sum_{i\in\calS}(m_i)^{1-\alpha}\right)^{\frac{1}{1-\alpha}}+\sum_{i\not\in\calS}m_{i}\right).
\end{split}
\end{equation}
Then, we know \btk robustness condition will be filled if the \Renyi divergence do not exceed the above value. 
\end{proof}

\subsection{Proofs to the claims used in the proof of Theorem~\ref{theo:renyi_btk}}\label{sec:missing}
$(i)$ [Natural] To reach the minimum, there is exactly $k_0$ different components in the top-$k$ of $\bbE[\ourbm]$ and $\bbE[\btm]$.\\
\begin{proof}
Assume that ${i_1},\cdots,{i_{k_0+j}}$ are the components not in the top-$k$ of $\ourbm$ but in the top-$k$ of $\btm$. Similarly, we let ${i_1'},\cdots,{i_{k_0+j}'}$ to denote the components in the top-$k$ of $\ourbm$ but not in the top-$k$ of $\btm$. Consider we have another $\btm^{(2)}$ with the same value with $\btm$ while $\tm_{i_{k_0+j}}$ is replaced by $\tm_{i_{k_0+j}'}$. In other words, there is $k_0+j$ displacements in the top-$k$ of $\btm$ while there is $k_0+j-1$ displacements in the top-$k$ of $\btm^{(2)}$.
Thus,
\begin{equation}\nonumber
\begin{split}
&s\left(\bbE[\ourbm]\,|| \;\bbE[\btm^{(2)}]\right)-s\left(\bbE[\ourbm]\,|| \;\bbE[\btm]\right)\\
=\;&\left( \frac{\left(\tm_{i_{k_0+j}'}\right)^{\alpha}}{\left(m_{i_{k_0+j}}\right)^{\alpha-1}} + \frac{\left(\tm_{i_{k_0+j}}\right)^{\alpha}}{\left(m_{i_{k_0+j}'}\right)^{\alpha-1}}\right)-\left( \frac{\left(\tm_{i_{k_0+j}}\right)^{\alpha}}{\left(m_{i_{k_0+j}'}\right)^{\alpha-1}} + \frac{\left(\tm_{i_{k_0+j}'}\right)^{\alpha}}{\left(m_{i_{k_0+j}}\right)^{\alpha-1}}\right).
\end{split}
\end{equation}
Because $m_{i_{k_0+j}} \geq m_{i_{k_0+j}'}$ and $\tm_{i_{k_0+j}} \leq \tm_{i_{k_0+j}'}$, we know $s\left(\bbE[\ourbm]\,|| \;\bbE[\btm^{(2)}]\right)-s\left(\bbE[\ourbm]\,|| \;\bbE[\btm]\right)\leq 0.$
Thus, we know reducing the number of misplacement in top-$k$ can reduce the value of $s\left(\bbE[\ourbm]\,|| \;\bbE[\btm]\right)$. If requires at least $k_0$ displacements, the minimum of $s\left(\bbE[\ourbm]\,|| \;\bbE[\btm]\right)$ must be reached when there is exactly $k_0$ displacements.
\end{proof}

$(ii)$ To reach the minimum, $\tm_{k-k_0},\cdots,\tm_k$ are not in the top-$k$ of $\bbE[\btm]$.\\
\begin{proof}
Assume that ${i_1},\cdots,{i_{k_0}}$ are the components not in the top-$k$ of $\ourbm$ but in the top-$k$ of $\btm$. Similarly, we let ${i_1'},\cdots,{i_{k_0}'}$ to denote the components in the top-$k$ of $\ourbm$ but not in the top-$k$ of $\btm$. Consider we have another $\btm^{(2)}$ with the same value with $\btm$ while $\tm_{i_{j}}$ is replaced by $\tm_{j'}$, where $\tm_{j'}$ is in the top-$k$ of $\btm$ and $m_{j'} \geq m_{i_{j}}$. In other words, $\tm_{j'}^{(2)}$ is no longer in the top-$k$ components of $\btm$ while $\tm_{i_j}^{(2)}$ goes back to the top-$k$ of $\tm_{i_j}^{(2)}$. Note again that $m_{j'} \geq m_{i_{j}}$ and $j' \leq i_{j}$. Thus,
\begin{equation}\nonumber
\begin{split}
&s\left(\bbE[\ourbm]\,|| \;\bbE[\btm^{(2)}]\right)-s\left(\bbE[\ourbm]\,|| \;\bbE[\btm]\right)\\
=\;&\left( \frac{\left(\tm_{i_{j}}\right)^{\alpha}}{\left(m_{j'}\right)^{\alpha-1}} + \frac{\left(\tm_{j'}\right)^{\alpha}}{\left(m_{i_{j}}\right)^{\alpha-1}}\right)-\left( \frac{\left(\tm_{{j'}}\right)^{\alpha}}{\left(m_{j'}\right)^{\alpha-1}} + \frac{\left(\tm_{i_{j}}\right)^{\alpha}}{\left(m_{i_{j}}\right)^{\alpha-1}}\right).
\end{split}
\end{equation}
Note that $\tm_{i_j} \leq \tm_{i}$, we have $s\left(\bbE[\ourbm]\,|| \;\bbE[\btm^{(2)}]\right)-s\left(\bbE[\ourbm]\,|| \;\bbE[\btm]\right)\geq 0$. Thus, we know that ``moving a larger component out from top-$k$ while moving a smaller component back to top-$k$'' will make $s(\cdot ||\cdot)$ larger. Then, $(ii)$ follows by induction.
\end{proof}

$(ii')$ To reach the minimum, $\tm_{k+1},\cdots,\tm_{k+k_0+1}$ must appear in the top-$k$ of $\bbE[\btm]$, which holds according to the same reasoning as $(ii)$.


\subsection{Formal Version of Theorem~\ref{theo:concentration} with Proof}\label{sec:app:concen}
Before presenting Theorem~\ref{theo:concentration}, we first provide a technical lemma for the concentration bound for $\phi(\ourbm) =  2k_0\left(\frac{1}{2k_0}\sum_{i\in\calS}(m_{i^*})^{1-\alpha}\right)^{\frac{1}{1-\alpha}}-\sum_{i\in\calS}m_{i^*}$. Considering that $||\ourbm||_1 = 1$, we know $\epsilon_{\text{robust}} = -\ln\left(1+\phi(\ourbm)\right)$. To simplify notation, we let $\psi(\ourbm) = \left(\frac{1}{2k_0}\sum_{i\in\calS}(m_{i^*})^{1-\alpha}\right)^{\frac{1}{1-\alpha}}$. Thus, we have $\phi(\ourbm) =  2k_0\psi(\ourbm)-\sum_{i\in\calS}m_{i^*}$
\begin{lemma}\label{lem:tech}
Using the notations above, we have:
$$\Pr\left[\phi(\ourbm)\leq\phi(\hat{\ourbm})+\delta\right]\leq 1-n\cdot\exp\left(-2T\delta^2\left[\psi^{\alpha}(\hat{\ourbm})\left(\sum_{i\in\calS}\hat{m}_{i^*}^{-\alpha}\right)-2k_0\right]^2\right).$$
\end{lemma}
\begin{proof} 
We first prove the following statement: if for all $i\in[n]$, $|\hat m_i - m_i|\leq \delta$, we always have:
\begin{equation}\label{equ:first}
\phi(\ourbm)-\phi(\hat{\ourbm})\leq \delta\left[\psi^{\alpha}(\hat{\ourbm})\left(\sum_{i\in\calS}\hat{m}_{i^*}^{-\alpha}\right)-2k_0\right].
\end{equation}
Because $\phi(\cdot)$ is a concave function when all components of the input vector are positive. Thus, if $\phi(\ourbm)\geq\phi(\hat{\ourbm})$ we have,
$$\phi(\ourbm)-\phi(\hat{\ourbm}) \leq \sum_{i\in\calS}(m_{i^*}-\hat m_{i^*})\cdot\left.\frac{\partial\phi(\ourbm)}{\partial m_{i^*}}\right|_{\ourbm = \hat{\ourbm}}.$$
Then, inequality (\ref{equ:first}) follows by the fact that $\frac{\partial\phi(\ourbm)}{\partial m_{i^*}} = m_{i^*}^{-\alpha}\cdot\psi^{\alpha}(\ourbm)$.

Then, according to Hoeffding bound, we have:
$$\Pr\left[\forall i\in[n],\;|\hat m_i - m_i|\leq \delta\right] \geq 1-n\cdot\exp\left(-2T\delta^2\right).$$
Because $\left[\forall i\in[n],\;|\hat m_i - m_i|\leq \delta\right]$ is a special case of $\left[\phi(\ourbm)\leq\phi(\hat{\ourbm})+\delta\right]$, Lemma~\ref{lem:tech} follows.
\end{proof}
Then, we formally present Theorem~\ref{theo:concentration}:

\noindent{\bf Theorem~\ref{theo:concentration} (formal)} {\em Using the notations above, we have,
\begin{equation}\nonumber
\begin{split}
    &\Pr\left[\epsilon_{\text{robust}}\geq\hat\epsilon_{\text{robust}}-\delta\right]\\
    \leq\;& 1-n\cdot\exp\left(-2T(e^{-\delta}-1)^2(\phi(\hat\ourbm)-1)^2\left[\psi^{\alpha}(\hat{\ourbm})\left(\sum_{i\in\calS}\hat{m}_{i^*}^{-\alpha}\right)-2k_0\right]^2\right).
\end{split}
\end{equation}} 
\begin{proof}
Theorem~\ref{theo:concentration} follows by applying $\epsilon_{\text{robust}} = -\ln\left(1+\phi(\ourbm)\right)$ to Lemma~\ref{lem:tech}.
\end{proof}




\end{document}